\documentclass[article]{eptcs}
\usepackage{breakurl}            
\usepackage{underscore}           

\usepackage{times}
\usepackage{helvet}
\usepackage{courier}
\usepackage{amsmath}
\usepackage{amsthm}
\usepackage{amssymb}
\usepackage{cite}
\usepackage[curve]{xypic}

\newtheorem*{theorem*}{Theorem}
\newtheorem*{proposition*}{Proposition}
\newtheorem*{lemma*}{Lemma}
\newtheorem{theorem}{Theorem}[section]
\newtheorem{definition}[theorem]{Definition}

\newtheorem{proposition}[theorem]{Proposition}
\newtheorem{corollary}[theorem]{Corollary}
\newtheorem{example}[theorem]{Example}

\newcommand{\vset}{\vec X}
\newcommand{\structure}{{\cal S}}
\newcommand{\exo}{{\cal U}}
\newcommand{\ens}{{\cal V}}

\newcommand{\range}{{\cal R}}

\newcommand{\graph}{{\sf G}}
\newcommand{\effect}{{\cal F}}
\newcommand{\model}{M}
\newcommand{\mods}{{\cal M}}
\newcommand{\assign}[2]{{#1}\leftarrow{#2}}
\newcommand{\surge}[2]{[{#1}\leftarrow{#2}]}
\newcommand{\cause}{{\sf Cause}}
\newcommand{\butfor}{{\sf ButFor}}
\newcommand{\crel}[2]{\succeq_{#1,#2}}
\newcommand{\halpern}{{\sf HP}}

\newcommand{\noo}[1]{}

\title{On Preemption and Overdetermination in Formal Theories of Causality}
\author{Sjur K Dyrkolbotn
\institute{Western Norway University of Applied Sciences\\ Bergen, Norway}
\email{sdy@hvl.no}
}
\def\titlerunning{Preemption and Overdetermination}
\def\authorrunning{S.K. Dyrkolbotn}
\begin{document}
\maketitle

\begin{abstract}
One of the key challenges when looking for the causes of a complex event is to determine the causal status of factors that are neither individually necessary nor individually sufficient to produce that event. In order to reason about how such factors should be taken into account, we need a vocabulary to distinguish different cases. In philosophy, the concept of overdetermination and the concept of preemption serve an important purpose in this regard, although their exact meaning tends to remain elusive. In this paper, I provide theory-neutral definitions of these concepts using structural equations in the Halpern-Pearl tradition. While my definitions do not presuppose any particular causal theory, they take such a theory as a variable parameter. This enables us to specify formal constraints on theories of causality, in terms of a pre-theoretic understanding of what preemption and overdetermination actually mean. I demonstrate the usefulness of this by presenting and arguing for what I call the principle of presumption. Roughly speaking, this principle states that a possible cause can only be regarded as having been preempted if there is independent evidence to support such an inference. I conclude by showing that the principle of presumption is violated by the two main theories of causality formulated in the Halpern-Pearl tradition. The paper concludes by defining the class of empirical causal theories, characterised in terms of a fixed-point of counterfactual reasoning about difference-making. It is argued that theories of actual causality ought to be empirical.
\end{abstract}

\section{Introduction}

When reasoning about actual causality, it is easy to fall in the logic trap, by thinking that causes behave like necessary and sufficient conditions for outcomes.\footnote{Actual causality is concerned with causal attributions for concrete events, not general dependencies or regularities. To illustrate with an example from Halpern \cite{halpern15}, 'smoking caused Bill's cancer' is a claim about actual causality, while 'smoking causes cancer', a statement of general causality, is not. See, generally, \cite{lewis73,halpern05}.} Of course, we know this way of thinking is misguided. If Jane and Julie both give a lethal dose of poison to Bob, we realise that they are both causes of his death, even though neither of them are necessary. If they give half a dose each, the same conclusion is drawn, even if neither is sufficient.

This is widely accepted and hardly any theory of causality proposes otherwise. But the trappings of classical logic are not so easily avoided. In legal theory, for instance, the so-called NESS test holds that $A$ is cause of $B$ if $A$ is a {\it necessary element of a sufficient set of conditions} for $B$ \cite{wright88}. This gloss on the logic-based perspective suffices to deal with simple examples, but problems still arise.\footnote{For a formalisation of the NESS test that suffers from many of the same problems as those identified with respect to the Halpern-Pearl definitions considered in the present paper, see \cite{braham12}.}

For an illustration of this, consider a scenario where Jane gives Bob half a dose of poison while Julie administers a whole dose. We know from experiment (with other Bobs, no doubt) that one dose or more is deadly. Furthermore, we assume that any integer percentage of a whole dose of poison could have been given by both Jane and Julie; none of them were compelled to give a certain dose of poison to Bob. When Bob dies, we would like to know who caused his death. In view of logic-based intuitions, some might be tempted to claim that only Julie counts as a proper cause, since her dose of poison was sufficient all by itself. Indeed, this is the judgement made by the NESS test: Jane is not necessary for any sufficient set of conditions, since such a set also has to include Julie.

Others will disagree with this judgment, maintaining that we have to count Jane as a proper cause, since she clearly contributed to the outcome. This is the position taken by the first two definitions of causality in the HP tradition, due to Halpern and Pearl \cite{halpern05}. However, the most recent definition in this tradition, due to Halpern, sides with the NESS test; Jane will not be considered a cause, or even part of cause, when Bob dies \cite{halpern15,halpern16a}.

This underscores a lack of theoretical agreement about very simple cases, suggesting the need for work that can help us reason {\it about} causal theories, to help us move beyond vague intuitions when trying to evaluate them.\footnote{In \cite{glymour10}, formal theories of actual causation are criticised for being based on ``induction from intuitions about an infinitesimal fraction of the possible examples and counterexamples''. To improve the methodology, the authors propose an approach where the aim is to find ``reliable indicators'' for actual causation, rather than sweeping formal definitions. While I agree with the gist of the criticism, the present paper does not pursue a pragmatic response, but aims instead to contribute to a meta-theory that can be used to justify formal definitions of causality in a more systematic way than by testing them on a handful of examples. I believe we should aim to develop such a theory to account for some agreed-upon or defeasibly stipulated principles that express what we expect of a formal theory of causation in a given context.}

To highlight what is at stake here, consider changing the scenario above so that Jane only gives Bob 1\% of a full dose of poison. Should she still count as a cause of his death? Then contrast this with the situation that arises when Jane gives Bob 99\% of a full dose of poison - enough to kill him even if Julie had given only 1\%. Is it still defensible to deny that Jane was a cause of Bob's death? Intuitively, most of us would answer no to both questions, but this forces us to acknowledge that our logic-based intuitions are at best context-dependent: if Jane gives Bob a 99\% dose of poison, we are not willing to deny that Jane counts as a cause of Bob's death, even though we did want to do this when her contribution was 1\% of a full deadly dose. From this arises the problematic question of where exactly to draw the line: how much poison does it take to count as a cause of Bob's death by poison? If we go down that route, our theory of causality will effectively stumble over a {\it sorites} paradox, which is not likely to be solved any time soon. It might be better, therefore, to understand actual causes as factors that make a {\it causal contribution} to a given outcome.

In the following paper, this is the concept of ``cause'' I will elaborate on, building on how that notion is traditionally used in the HP framework.\footnote{In recent work, Halpern has used the terminology ``part of a cause'' and ``complex cause'' to highlight a distinction that arises quite naturally from his most recent definition of causality \cite{halpern15}. This does not address the conceptual problem of how to draw a distinction between ``proper'' causes and causal contributions, if such a distinction is thought to exist. In the Jane-Julie example, Jane is neither a complex cause nor a part of a cause, meaning that the intuition that she contributes to Bob's death by poisoning him is not captured at all. Moreover, if Jane gives Bob 1\% of a full dose of poison while Julie gives 99\% of a full dose, Jane alone will suddenly count as a complex cause of Bob's death according to Halpern's definition. Hence, if we want to resist attributing causal significance to Jane's action in such cases, we cannot simply equate Halpern's notion of a complex cause with that of a cause {\it simpliciter}.} With the understanding that ``cause'' means ``causal contribution'', it becomes more plausible to regard Jane as a cause of Bob's death in the example above. In principle, this intuition about causal contribution also remains robust regardless of whether she administers 1\% or 99\% of a full dose of poison. Either way, Jane contributes to Bob's death. It might still be possible to dispute this, but the more obvious objections have lost their appeal.

Moreover, when we consider the NESS test or the most recent definition of HP causality more closely, we notice something strange: in the situation described above, these definitions fail to recognise the contribution made by Jane, but in a model where Julie is assumed to have exactly one alternative action, namely to give half a dose of poison to Bob, Jane will come out as a cause of his death after all. Why should the causal status of Jane's actual action depend on the contrary-to-fact options available to Julie? The natural starting-point is to assume that there is no such dependence, unless we have a specific reason to believe otherwise.\footnote{This could be, for instance, if the outcome under consideration is {\it inevitable} due to Julie's limited options, in which case neither June nor Julie will be regarded as having caused Bob's death.} In the June and Julie example, I fail to see any such reason. To my knowledge, none has been produced in the literature either, where these kinds of cases appear to have been overlooked. 

To structure our reasoning about the problem that arises here, and other problems like it, it is helpful to use the concepts of {\it overdetermination} and {\it preemption}, broadly construed. The former notion is typically used to pinpoint causes that are not necessary for their effects, while the second notion is generally used to pinpoint putative causes that are blocked from making a contribution to their apparent effects.\footnote{See, e.g.,  \cite{lewis73,hitchcock11,paul13,bernstein14}.} Understood in this way, when Jane counts as a cause of Bob's death, she is an overdetermining cause, adding to the contribution made by Julie. By contrast, if Jane is not regarded as having made a causal contribution, she must have been preempted from doing, presumably by Julie, whose dose was sufficient all by itself. This usage is consistent with how I will define these terms in the present paper, as theory-neutral concepts that can be used to talk about causal theories. The main contribution of my work is that I show how this can be done formally in the HP framework, leading to a template for defining a new class of causal theories within that framework, characterised by a fixed-point property.


\section{Causal Models and Theories of Causation}\label{sec:2}

I begin by presenting the basic constructs of the HP framework for causal modelling \cite{halpern01,halpern05,halpern15}. This framework is based on formal specifications that generate a set of structural equations that represent causal dependencies. To specify the variables and their ranges we use a {\it causal structure}, a tuple $\structure = (\exo,\ens,\range)$ where $\exo$ is a finite set of exogenous variables, $\ens$ is a finite set of endogenous variables and $\range: \exo \cup \ens \to 2^{D}_{\sf {fin}}$ is a map from variables to a finite collection of values from some domain $D$.

Given a causal structure, a causal model over $\structure$ is a tuple $\model = (\graph,\effect, \vec u)$ where $\graph = (N,\ens)$ is a directed acyclic graph (DAG) over $\ens$, with directed edges $N \subseteq \ens \times \ens$. For all $V \in \ens$, $p(V) = \{U \mid (U,V) \in N\}$ collects the set of parents of $V$ in $\graph$, i.e., the set of variables pointing to $V$. $\effect = (F_V)_{V \in \ens}$ is a collection of functions, one for each endogenous variable. For all $V \in \ens$, $F_V: \prod\limits_{U \in \exo \cup p(V)} \range(U) \to \range(V)$ depends only on $\exo$ and the values of all parents of $V$. The final parameter of a model, $\vec u = (u_1,\ldots,u_n)$, is a setting of values for $\exo = \{U_1,\ldots,U_n\}$, ordered arbitrarily such that $u_i \in \range(U_i)$ for all $1 \leq i \leq n$.

\begin{example}\label{ex:1}
	To represent two of the Jane-Julie scenarios considered in the introduction, we can use $\structure_1 = (\exo_1,\ens_1,\range_1)$ and $\structure_2 = (\exo_2,\ens_2,\range_2)$ where $\exo_1 = \exo_2 = \{U_1,U_2\}$ and $\ens_1 = \ens_2 = \{J_1,J_2,B\}$. The ranges are defined such that $\range_1(U_1) = \range_2(U_1) = \range_1(J_1) = \range_2(J_1) = \{0,0.5\}, \range_1(U_2) = \range_1(J_2) = \{0.5,1\}, \range_2(U_2) = \range_2(J_2) = \{0,0.5,1\}$ and $\range_1(B) = \range_2(B) = \{0,1\}$. So the only difference between $\structure_1$ and $\structure_2$ concerns the range of $U_2$ and $J_2$. We can then model the two problem cases by two models, $M_1$ and $M_2$, which are identical except for the range of $U_2, J_2$. Specifically, both models are made up of the context $\vec u = (0.5,1)$ and the graph $\graph$ depicted on the left below, with associated functions shown on the right (domains and function names are left implicit).
	$$
	\begin{array}{cc}
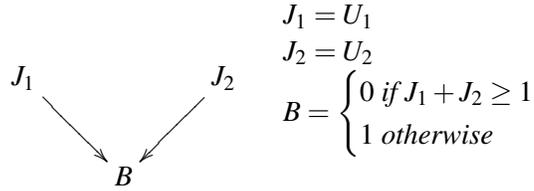
\xymatrix{ J_1 \ar[dr] && J_2 \ar[dl] \\ 
									&	B	 } & \begin{array}{l} 
									J_1 = U_1 \\
									J_2 = U_2 \\
									B = \begin{cases} 0 \text{ if } J_1 + J_2 \geq 1 \\
									1 \text{ otherwise }
									\end{cases}
									\end{array}
\end{array}
	$$
Intuitively, $J_1$ records the amount of poison given by Jane, $J_2$ is the amount given by Julie, and $B$ is a variable encoding whether Bob lives or dies, with $0$ representing that he dies.
\end{example}

If $X \subseteq \exo \cup \ens$ is a set of variables, we use $[\assign{\vec X}{\vec x}]$ to denote an assignment of values to $X$. Here $\vec X = \{X_1,\ldots,X_n\}$ is an arbitrary ordering of $X$ and $\vec x = (x_1,\ldots,x_n)$ with $x_i \in \range(X_i)$ for all $1 \leq i \leq n$. We write $(\vec X = \vec x)$ to denote the conjunction $\bigwedge\limits_{1 \leq i \leq n}(X_i = x_i)$. Moreover, we may overload this notation by writing $[\assign{\vec Y}{\vec x}]$ and $(\vec Y = \vec x)$ when $\vec Y \subseteq \vec X$. In these cases, we read $\vec x$ as $(x_i)_{\{1 \leq i \leq n \mid X_i \in \vec Y\}}$. Similarly, we may write $[\assign{\vec X}{\vec x}]$ also when $\vec x$ is a partial assignment on $\vec X$, i.e., such that there is $X_i \in \vec X$ with $x_i \not \in \vec x$. In this case, $X_i$ is not affected by the assignment $\vec x$.

Given a model $\model$, the associated system of equations is $(V = F_V(\vec X_V))_{V \in \ens}$ where every $\vec X_V = (U_1,\ldots,U_n)$ is an arbitrary enumeration of the elements in $\exo \cup p(V)$. Since $\graph$ is acyclic, this system of equations has a unique solution, namely $S_\model: \ens \to D$ (where, obviously, $S_\model(V) \in \range(V)$ for all $V \in \ens$).

An atomic causal expression has the form $(V = v)$ where $V \in \ens$, $x \in \range(V)$. It is to be understood as the claim that the variable $V$ has the value $v$. Given a model $\model$ and an atomic expression $(V = v)$, truth on a model is defined by the clause $\model \models (V = v) \Leftrightarrow S_M(V) = v$.

A boolean combination of atomic expressions is called a {\it basic formula}. Truth on a model is inductively defined for boolean combinations in the standard way.


A {\it formula} is either a basic formula or an update-formula of the form $\surge {\vec X} {\vec x} \phi$ where $\phi$ is a basic formula and $\surge {\vec X}{\vec x}$ is an assignment of values to some $\vec X \subseteq \ens$. An update-formula is to be understood as saying that $\phi$ is guaranteed to be true if we intervene to ensure that $x_i$ is assigned to $X_i$ for all $x_i \in \vec x$.

Truth is defined for update-formulas by the recursive clause 
$$\model \models \surge {\vec Y}{\vec y}\phi \Leftrightarrow M_{\surge{\vec Y}{\vec y}} \models \phi$$  where $M_{\surge{\vec X}{\vec x}}$ is the model obtained from $M$ by replacing $F_{X_i}$ with the constant function $x_i$ for all $x_i \in \vec x$. We will not require nested update-formulas, but we will work with sequences of model updates. Specifically, if  $M' = M_{[\assign{\vec X}{\vec x}]}$, we might ask about the truth of formulas like $\psi = [\assign{\vec Y}{\vec y}]\phi$ on $M'$. This is well-defined also when $Y \cap X \not = \emptyset$; to evaluate the truth of $\psi$ on $M'$ one simply replaces $x_i$ by $y_j$ for all $X_i = Y_j \in X \cap Y$.

\begin{example}
	[Example \ref{ex:1} continued]
	Recall the models $M_1$ and $M_2$ from Example \ref{ex:1}. Clearly, the solutions of the associated systems of equations are the	same, $S_{M_1} = S_{M_2} = \{(J_1,0.5),(J_2,1),(B,0)\}$. Hence, we have $M_1 \models (B = 0)$ and $M_2 \models (B = 0)$. Moreover, we have $M_1 \models [\assign {J_1}{0}](B=0)$ and $M_2 \models [\assign{J_1}{0}](B=0)$. This shows that $J_1$ taking the value $0.5$ is not a necessary condition for $(B = 0)$ in either $M_1$ or $M_2$. Also observe that while $M_2 \models [\assign{J_2}{0}]\neg (B=0)$, the update formula $[\assign{J_2}{0}]\neg (B=0)$ is not defined with respect to $M_1$, since $0 \not \in \range_1(J_2)$.
\end{example}

Given a model $\model$, we define a set of associated models as follows:
$$
\mods = \{M' \mid \exists \vec X \subseteq \ens: \exists \vec x \in \range(\vec X): M_{[\assign{\vec X}{\vec x}]} = M'\}
$$
That is, $\mods$ collects all causal models that can be obtained from $\model$ by intervening on some of its endogenous variables. For instance, the models $M_1$ and $M_2$ from Example \ref{ex:1} both admit a variant of the model depicted below as one of its associated models.

$$
\begin{array}{cc}
\xymatrix{ J_1 & J_2 & B } &  \begin{array}{l} J_1 = U_1 \\
											J_2 = U_2 \\
											B = 1
							\end{array}
\end{array}
$$

If the context is $\vec u = (0.5,1)$, this is the counterfactual scenario where Bob lives even though he receives one and a half doses of poison. Hence, the graph has no edges at all. If the range of the variables is given by $\structure_1$, the model above is associated with $M_1$ and when the range is determined by looking at $\structure_2$, the model is associated with $M_2$. Either way, the relevant update that generates the no-edges model is $[\assign{B}{1}]$.

The point of causal models is to facilitate the definition of causal theories. In the HP tradition, there have been many proposals  for such theories, suggesting an indirect analysis of their properties. For this purpose, we first define the basic signature of a causal theory in this setting.

\begin{definition}[Causal theories]\label{def:causal}
Given a signature $\structure$, a causal theory for $\structure$ is a map $\cause$, such that $\cause(M,\phi) \subseteq \ens$ for all models $M$ and all basic formulas $\phi$ built over $\structure$.
\end{definition}

As a first example of a causal theory, we now define the bedrock of all other theories of actual causality in {\halpern} tradition: the {\it but-for} test.

\begin{definition}[But-for causality]\label{def:butforcause}
The but-for theory, $\butfor$, is defined for all models $M$ and all basic formulas $\phi$ as follows:
	$$
	X \in \butfor(M,\phi) \Leftrightarrow M \models \phi \text{ and } \exists x \in \range(X): M \models\surge{X}{x} \neg \phi
	$$
\end{definition}

It is hard to see how the but-for test could be wrong when it designates something as a causal contributor. The definition can be given a logical reading, but it effectively asks for empirical evidence of $V$'s causal power: there is a way to change $V$ so that $\phi$ no longer obtains. The problem with the but-for theory is that it does not recognise enough causes. 

To illustrate, let $M_1$ and $M_2$ be as in Example \ref{ex:1}. Then it is easy to see that $\butfor(M_2,(B=0)) = \{J_2, B\}$. We have $M_1 \models [\assign{J_2}{0}]\neg (B = 0)$, so $J_2$ is a cause, while we also have $M_2 \models [\assign{J_1}{0}](B = 0)$, so $J_1$ is not a cause. However, $\butfor(M_1,(B = 0))$ only contains the trivial cause $B$, since $M_1 \models [\assign{J_2}{0.5}](B = 0)$. This is unreasonable; clearly, poison is causing Bob's death in both these models. Moreover, his death is not inevitable, since $M_1 \models [\assign{J_1}{0},\assign{J_2}{0.5}]\neg B$. Hence, it seems we need to be more permissive about what we regard as causes. Many  definitions have been formulated to achieve this in the HP framework, with Halpern's most recent proposal given below.

\begin{definition}[HP causality \cite{halpern15}]\label{def:hpcause}
	Given a model $M$ and basic formula $\phi$, $\vec X$ is a complex cause\footnote{Strictly speaking, the cause emanates from the value of $\vec X$ in $M$, but for the definition in \cite{halpern15} we might as well abstract directly to the level that holds more interest.} for $\phi$ in $M$ if 
	\begin{description}
		\item[AC1.] $M \models \phi$.
		\item[AC2.] There is a set $\vec W \subseteq \ens$ and a setting $\vec x$ for $\vec X$ such that
		$$
		M \models [\assign{\vec W}{\vec w}, \assign{\vec X}{\vec x}] \neg \phi	
		$$
		where $M \models (\vec W = \vec w)$.
		\item[AC3.] $\vec X$ is minimal with respect to set-inclusion; there is no subset of $\vec X$ for which AC1 and AC2 holds.
	\end{description}
	The notion of a complex cause gives rise to the following causal theory $\halpern$, where we maintain the cause-as-contribution perspective, defined for all $M, \phi$ as follows:
	$$
	\begin{array}{r}
	\halpern(M,\phi) = \{V \in \ens \mid \exists \vec X: V \in \vec X \text{ and } \\ \vec X \text{ is a complex cause of } \phi \text{ in } M\}
	\end{array}
	$$
\end{definition}

If $\vec x, \vec W$ and $\vec w$ can be used to show that $\vec X$ satisfies AC1-AC3 for some $\phi$, we say that $(\vec x, \vec W, \vec w)$ is a {\it witness} of $\vec X$ being a complex cause of $\phi$. It is then also a witness of $V \in \halpern(M,\phi)$ for all $V \in \vec X$.

For $M_1$, the $\halpern$ theory does better than the but-for test. Specifically, since $M_1 \models [\assign{J_1}{0},\assign{J_2}{0.5}]\neg (B = 0)$, it is easy to verify that $\{J_1,J_2\}$ is a complex cause of $(B = 0)$ in $M_1$. Hence, we get $\halpern(M_1,(B = 0)) = \{J_1,J_2,B\}$, which is intuitively reasonable when we come from a cause-as-contribution starting point. After all, Bob dies of poison, as administered to him by both Jane and Julie. However, moving to $M_2$ we encounter the problemtic case, where $\halpern$ agrees with the but-for theory. Specifically, $\{J_1,J_2\}$ fails to satisfy the minimality requirement, since assigning $0$ to $J_2$ suffices to ensure $\neg (B = 0)$. Hence, $J_1$ is no longer considered a cause. 

\section{Preemption, Overdetermination and the Principle of Presumption}\label{sec:3}

In my terminology, the Jane-Julie example points to an anomaly regarding how the $\halpern$ theory draws the line between overdetermination and preemption. To make this formally precise, the first step is to extract from the {\halpern} theory how counterfactual reasoning is used to analyse causes. The scope of permitted interventions, taking us away from the actual state that yields $\phi$, is made clear in AC2 and AC3. Combined, these conditions effectively say that we may intervene at will on the elements of complex causes of $\phi$, while interventions outside these sets must be limited to keeping variables fixed at their actual values.\footnote{We will make this formally precise in a moment.} How to justify this is a matter I will not get into here, except by noting that it can be seen as an application of the so-called {\it similarity principle}, whereby counterfactuals need to be evalutated relative to states of affairs that are as close as possible to the actual state of affairs.\footnote{See generally \cite{lewis73}.} To intervene by changing the values of complex causes is necessary to apply but-for reasoning, to obtain evidence of difference-making. Meanwhile, keeping the values of other variables fixed can be justified directly as an effort to keep the counterfactual scenario as close as possible to the actual one.\footnote{This glosses over deeper issues about whether our insistence on similarity of particular facts -- the values of variables -- risks undermining deeper similarities, namely compliance with the equation associated with each variable (an equation that can be overruled by an intervention that keeps a variable fixed at its actual value). However, I will not get into this problem here; I do not see how my approach does much to resolve it one way or the other.}

Encoding this formally, we now define the collection of models obtained by intervening only on {\it some} variables from a specific set $\vset$, while keeping {\it some} variables from its complement, $c(\vset)$, fixed at their actual values. A conceptual premise of the {\halpern} definition is that interventions of this kind span a space of counterfactuals that we need to consider in order to reason about causality. To understand the formalisation below, recall that $[\assign{\vec V}{\vec v}]$ updates the equation for all $V_i \in \vec V$ with the constant value $v_i \in \vec v$ while leaving the equations associated with other variables $V_j \in V$, for which there is no corresponding $v_j \in \vec v$, unchanged. That is, we quantify over {\it partial} interventions, allowing considerations of counterfactuals where an arbitrary set of variables is left undisturbed. These variables are free to change compared to their actual value, if their equations so dictate in view of interventions elsewhere in the system. Effectively, this can be understood as encoding a form of {\it uncertainty} about how to understand the similarity-constraint, forcing us to consider a larger space of contrary-to-fact models, defined as follows:

\begin{equation}\label{vset}
\mods(\vset) = \{M' \mid \exists \vec v \in \range(\vset), \exists \vec w \in \range(c(\vset)): M_{[\assign{\vset}{\vec v},\assign {c(\vset)}{\vec w}]} = M' \text{ and } M \models (c(\vset) = \vec w)\}
\end{equation}
Intuitively, $\mods(\vset)$ collects all those models that can be obtained from $M$ by assigning arbitrary range-permitted values to some arbitrary subset of $\vset$, while holding an arbitrary set of variables from $c(\vset)$ fixed at their actual value in $M$.

We are now ready to provide a formal definition of overdetermination and preemption, as well as the notion of a {\it putative cause}. The definition is theory-neutral, taking an arbitrary causal theory as a parameter.

\begin{definition}\label{def:putative}
	Given a theory of causality, $\cause$, we say that:
	\begin{itemize}
		\item $V$ is an {\it overdetermining} cause of $\phi$ in $\model$ if $V \in \cause(M,\phi)$ and $V \not \in \butfor(M,\phi)$.
		\item $V$ is a {\it putative} cause of $\phi$ in $\model$ if $\exists \model' \in {\cal \model}(\cause(M,\phi))$ such that $S_M(V) = S_{\model'}(V)$ and $V \in \butfor(\model', \phi)$.
\item $V$ is a {\it preempted} cause of $\phi$ in $M$ if $V$ is a putative cause of $\phi$ in $M$, but not a cause of $\phi$ in $M$.

%
	\end{itemize}
\end{definition}

According to this definition, overdetermination occurs whenever the theory of causality recognises a cause of $\phi$ in $\model$ that does not satisfy the but-for test with respect to $(\model, \phi)$. This notion of overdetermination makes sense since it is paramterised by a causal theory. In standard usage, the notion of overdetermination picks out causes that fail to be necessary for the effects they contribute to \cite{schaffer14}. If we also have additional expectations about what overdetermining causes look like, we should keep in mind that the causal theory alone determines what we count as causes, c.f., Definition \ref{def:causal}. Hence, if we are serious about providing a theory-neutral definition, we cannot include further constraints in our definition of overdetermination. Additional constraints must be regarded instead as proposed {\it principles} of overdetermination, which may or may not bear further scrutiny.\footnote{For instance, it is common in the literature to say that overdetermination occurs when there are two or more {\it distinct} sets of causes that are both {\it sufficient} for the outcome \cite{bernstein16}. Clearly, no member of such a set is a but-for cause, so if our theory of causality classifies these elements as causes in the first place, Definition \ref{def:putative} correctly identifies them as overdetermining causes. However, if our theory of causality is more or less permissive, for instance by dropping the sufficiency requirement, the {\it theory} could diverge from the judgements intuitively expected. However, this is hardly evidence of a problem with our definition of overdetermination. Rather, it signals that the notion of overdetermination has been used elsewhere in a way that is {\it not} theory-neutral.}

The notion of a putative cause, and the resulting definition of preemption, is in need of a more substantive justification. Intuitively, the idea is that $V$ is a putative cause of $\phi$ if $V$ {\it would have been} a but-for cause of $\phi$ if some causes of $\phi$ had taken different values and some other variables had remained fixed. From an empirical perspective, $V$ will eventually reveal itself as a smoking-gun cause of $\phi$, provided you start a process of examination where you gradually ``remove'' more and more known causes of $\phi$.\footnote{Of course, {\it removing} a cause in this context means changing the value of a variable, not removing the whole variable from the system (which would amount instead to designing a new system). Hence, there might be many alternative values to try out, some of which seem intuitively irrelevant or misleading, for instance because these new values themselves could be {\it new} causes of $\phi$. In my view, this possibility does not raise any fundamental problems, although anomalous models can certainly be created by exploiting this phenomenon. However, I believe the impact of this is limited to highlighting the usefulness of the concept of normality developed in earlier work by Halpern and others. This concept is based on the acknowledgement that we might have to refine our application of contrary-to-fact causal theories by restricting attention to a certain subset of the counterfactual domain, e.g., to consider only values we know to be either neutral or conducive to $\neg \phi$. A generalisation of the concepts presented in this article could then be given accordingly.}

This kind of approach corresponds closely with how we reason intuitively about causation when events {\it seem to be} overdetermined. For instance, when we justify that having voted for a winning candidate was causally significant, we direct attention to the {\it what if}-scenario where the winner enjoys less support, so that they would have lost if we had voted differently. Even if this very far from the actual distribution of votes, most of us are willing to accept this kind of make-believe as a sound argument for causal significance. If you voted for Trump in Pennsylvania or Florida, or even Texas, most of us would agree that you did indeed contribute to him becoming President. At least, this would be a natural starting point, justifying us in concluding that your vote was a putative cause.

But what if you voted for Trump in New York? After all, Trump lost the state of New York, so it seems wrong to say that those who voted for him there actually contributed to his victory. Indeed, even though there is an alternative distribution of votes that would have made you a but-for cause in this case as well, you are {\it no longer} classified as a putative cause according to my definition. The reason is that in order to make you a but-for cause, we would have to intervene also on the votes of people who voted for another candidate, changing those into votes for Trump. It seems clear that this is a highly questionable leap of make-believe to undertake when looking for the actual causes of Trump's victory. In keeping with the principle of similarity, it should not be permitted.

The most contentious aspect of how we use counterfactuals concerns the interventions we perform when we force non-causal variables to stay the same even if other changes {\it would have} made them different, according to the model. The issue that arises here can be illustrated by considering a recalcitrant voter $A$ who decides that he will always vote the exact opposite of $B$ (we assume $A$ has the necessary knowledge about $B$ to implement this scheme).\footnote{A famous Norwegian football coach has stated publicly that he implements this strategy vis-a-vis his brother.} Now, if $B$ votes Trump, is $B$ a cause of Trump winning? The answer does not seem obvious. For the purposes of this paper, I accept the {\halpern} analysis that takes $B$ to be a cause, witnessed by the counterfactual scenario where $A$ {\it does not} change his vote to cancel out $B$. 


As indicated already, my understanding of what counts as a putative cause tracks how the {\halpern} theory makes use of counterfactual reasoning. This might not be immidiately obvious from Definition \ref{def:hpcause}, so I record it as a proposition.

\begin{proposition}\label{prop:1}
For all models $M$, formulas $\phi$ and variables $V$, if $V \in \halpern(M, \phi)$, then $V$ is a putative cause of $\phi$ in $M$ with respect to $\halpern$.
\end{proposition}

\begin{proof}
Let $\vec X$ witness to the fact that $V \in \halpern(M, \phi)$. That is, $V \in \vec X$ and $\vec X$ is a minimal set of variables such that AC1 and AC2 holds, c.f., Definition \ref{def:hpcause}. It follows that $X \in \halpern(M, \phi)$ for all $X \in \vec X$. Let $\vec x, \vec W, \vec w$ witness to the fact that $\vec X$ satisfies AC1 and AC2. It then follows by minimality of $\vec X$ that all $X \in \vec X$ are but-for cause of $\phi$ in $M_1(X) = M_{[\assign{\vec X}{\vec x}, \assign{\vec W}{\vec w},\assign{X}{x}]}$ where $x$ is the actual value of $X$, i.e., $M \models (X = x)$. That is, $M_1(X)$ is the model witnessing to $\vec X$ being a complex cause of $\phi$ in $M$, except that $X \in \vec X$ is fixed at its actual value. Recall that $M \models (\vec W = \vec w)$ by AC2. Let $\vec F = \vec W \cap \halpern(M, \phi)$ collect all the elements of $\vec W$ that are \halpern-causes of $\phi$ in $M$. Define $\vec y$ such that it extends $\vec x$ by $w_i$ for all $W_i \in \vec F$. Moreover, define $\vec z$ as the restriction of $\vec w$ obtained by removing all $w_i$ for $W_i \in \vec F$. By construction, we then get $[\assign{\vec X}{\vec x}, \assign{\vec W}{\vec w}] = [\assign{\halpern(M,\phi)}{\vec y}, \assign{c(\halpern(M, \phi))}{\vec z}]$. Hence, $M_1(V) = M_{[\assign{\halpern(M,\phi)}{\vec y}, \assign{c(\halpern(M, \phi)}{\vec z}, \assign{V}{v}]}$ (where $M \models (V = v)$. Since $V$ is a but-for cause of $\phi$ in $M_1(V)$, this witnesses to the fact that $V$ is a putative cause of $\phi$ in $M$, c.f., Definition \ref{def:putative}.
\end{proof}

What this means is that the notion of putative causes faithfully captures an upper limit to what the $\halpern$ theory is willing to designate as a cause. Combined with the but-for test, which acts as a lower limit, we can now define a class of causal theories.

\begin{definition}\label{def:sim}
A causal theory is said to be {\it similarity-based} if for every $M, \phi$, it recognises at least all but-for causes of $\phi$ and at most all putative causes of $\phi$.
\end{definition}

In view of Proposition \ref{prop:1}, it follows that {\halpern} is a similarity-based causal theory. The but-for theory is clearly the unique minimal similarity-based causal theory. There are many distinct maximal theories, as the reader can verify. Interestingly, the \halpern-theory is {\it not} among them; there are putative causes that arise from the \halpern-theory without being designated as causes according to that theory. Or to put it differently, the \halpern-theory recognises preemption.

This is reasonable, as illustrated by the typical example of preemption in the literature: Suzy and Billy are both throwing a rock at a bottle. Suzy throws first and the bottle shatters; Billy throws second and he is accurate, so the bottle {\it would have} shattered if Suzy had missed or decided not to throw. However, we are not prepared to say that Billy is an actual cause of the bottle shattering; he is only the ``back-up'' cause. According to Definition \ref{def:putative}, applied to the standard {\halpern} analysis of the scenario, he is a preempted cause.

The model is depicted in Example \ref{ex:suzy} below. The important thing to note is that changing the value of $S$ to $0$, to reach the counterfactual model where Suzy does not throw her rock, is enough to make $B$ a but-for cause. Hence, if our theory regards Suzy as a cause of the bottle shattering, Billy's throw becomes a putative cause of the same. Under the {\halpern} analysis, the minimality constraint kicks in here, to block Billy from being regarded as an actual cause. To see this, note that if we keep $H_B$ fixed at its actual value, to encode that Billy's rock does not hit the bottle, then changing $S$ to $0$ is {\it sufficent} to change the outcome so that the bottle remains intact. In view of this, the set $\{S,B\}$, which would otherwise witness to $B$ being a cause, fails to satisfy the minimality constraint; according to the {\halpern} theory, $\{S\}$ is also a complex cause, and it is strictly smaller.

\begin{example}
	\label{ex:suzy}
Consider the model $M$ depicted below, with binary-valued variables and context $U_1=U_2=1$ (both Suzy and Billy throw a rock):
$$
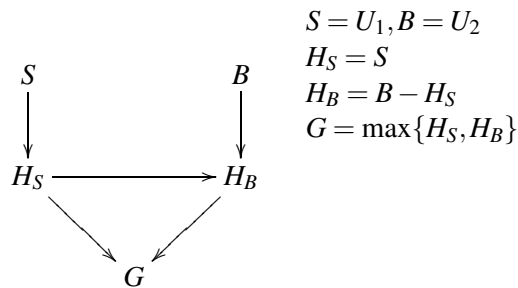
\begin{array}{cc}
	\xymatrix{ S \ar[d]&&  B \ar[d] \\
		H_S \ar[dr] \ar[rr] && H_B \ar[dl] \\
		& G &} & \begin{array}{ll}
		S = U_1, B = U_2 \\
		H_S = S \\ 
		H_B = B - H_S \\
		G = \max\{H_S,H_B\}
	\end{array}
\end{array}
$$
We observe that $M \models (G = 1)$, the bottle breaks. Moreover, we observe that there is no non-trivial but-for cause of $(G = 1)$. Specifically, changing either $S$ or $H_S$ will not change the outcome (since then Billy's throw shatters the bottle). The {\halpern} analysis therefore relies on keeping $H_B$ fixed at its actual value, to reach a counterfactual scenario where $S$ becomes a but-for cause of $(G = 1)$.
\end{example}

The {\halpern} analysis gives us the right result, but it lacks general appeal and has little or no explanatory power. As to the general appeal, what is the justification for the minimality constraint, when we know that the notion of necessity provides no secure footing for a theory of causality? While it is obvious that we need {\it some} constraint that blocks us from attributing causal relevance to irrelevant variables, I have not seen any argument that defends minimality as a principle of causal reasoning. As to explanatory power, the problem is that the minimality constraint forces us to analyse the case of Suzy and Billy in an unnatural way. To illustrate this, first note that the similarity constraint does not entitle us to presume that a counterfactual intervention to ensure that Billy misses the bottle is preferable to a counterfactual intervention whereby Billy does not throw at all. Quite the contrary, the standard {\halpern} model of this scenario encodes a world where Billy is necessarily accurate -- there is no setting of the exogeneous variables that results in Billy throwing and not hitting the bottle, unless Suzy has already broken it. In view of this, the intervention that makes Billy miss the bottle even though Suzy does not throw a rock creates a severe disparity between the counterfactual and the actual. Indeed, the state needed to verify that Billy is not a cause is {\it impossible} according to the model. Under the {\halpern} theory, therefore, we need to endorse the view that HP models are inaccurate, in order to explain preemption.

What could be a more reasonable justification? In my opinion, we do better when we replace minimality-talk by noting that $H_B$ counts as {\it evidence} of non-causation. Since Billy did not {\it actually} hit the bottle, he did not contribute to breaking it! This should be preferable to an explanation that invokes the spectre of necessity. Granted that this is true, is what I am proposing only a gloss on how we {\it describe} the workings of {\halpern} theory? I believe not, at least not when we follow our line of thought to the conclusion that {\it preemption requires proof}. What this means is that we should not designate putative causes as non-causes unless there is some variable whose actual value witnesses to the fact that the putative cause has been blocked.


This is exactly where the {\halpern} theory fails to make sense of the Jane-Julie example encoded in $M_2$ of Example \ref{ex:1}. We have already seen that Jane comes out as a preempted cause in this model, but where is the proof that she has been blocked? The only candidate is Julie, but the idea that Julie's action can count as evidence that Jane's poison has been blocked is hard to accept; all we know is that Julie is a but-for cause of Bob's death, while Jane is not. But this is a theory-dependent comparative difference; it gives us no reason to think that Julie is actually blocking Jane. If she is doing this, at least we seem entitled to expect a witness of such a mechanism in the form of an additional variable. As the {\halpern} theory fails to require this, I conclude that an evidential understanding of preemption, in line with the preferred explanation of what is happening in the Suzy-Billy example, is simply not consistent with the theory.

As mentioned already, the {\halpern} theory behaves particularly strangely on Jane-Julie examples, since it {\it does} regard Jane as a cause of Bob's death in some cases, when the range of Julie is restricted. Moreover, if we replace Julie by Mary and Mia, who give Bob 99\% of a full dose of poison each, then suddenly Jane is regarded as a cause of Bob's death (part of a cause in Halpern's terminology). This is so even though the contribution made by Jane seems less significant now, as a proportion of the total amount of poison given to Bob. Moreover, Mary and Mia together more than fulfill all those characteristics of Julie in $M_2$ that could perhaps justify regarding Jane as a preempted cause. 

The reason why minimality leads to a different result when Julie is replaced by Mary and Mia is that there are now three minimal sets satisfying AC1 and AC2 of Definition \ref{def:hpcause}. Two of these sets include Jane and one of Mary and Mia. The total amount of poison covered by these two sets is significantly less than that covered by Mary and Mia together. However, the fact that the sets including Jane are not strict subsets of the Mary-Mia set renders minimality impotent, so that Jane comes out as a cause after all. On an evidential understanding of what it means to designate something as a preempted cause, this difference in treatment appears unmotivated.

So far, we have only considered an example, suggesting that the {\halpern} theory fails to make sense of it. The principle at stake is that preemption requires proof, so the natural follow-up question is how this principle should be formalised so that it may be applied generally. Intuitively, we want to ensure that no putative cause is excluded unless there is some variable witnessing to the fact that the putative cause has been blocked. What does this mean? Since we identify causes by locating counterfactual worlds where they are but-for causes, a variable that blocks $V$ must be a variable that has to change value in order for $V$ to become a but-for cause. From this, it follows that we need an additional requirement to arrive at a sensible principle: we need to insist that a witness of non-causality must be found among those variables that are {\it not regarded as causes} of $\phi$ in the actual situation. If we do not insist on this, then our principle does not allow us to draw any meaningful distinction between overdetermination and preemption. Indeed, if $A$ and $B$ were classified as overdetermining causes, we would also be entitled to say that $B$ blocks $A$ and vice versa. If all the evidence we need to designate $A$ as a preempted cause of $\phi$ is that some variable has to change in order for $A$ to become a but-for cause, then all overdetermining causes are at risk of being classified as preempted ones.

The requirement that evidence of preemption must be grounded in non-causal variables is also needed to avoid circularity. If $A$ is a cause of $\phi$ and we allow it to function as the sole witness that a putative cause $B$ does not cause $\phi$, there is no theory-independent evidence to verify that our judgement about $B$ is correct. Hence, there is no theory-independent evidence to prove that our theory correctly adheres to the principle that preemption requires proof. To see this, note how $A$ could block $B$ merely by virture of causing $\phi$, according to our theory. This is what we did not want to accept in the Jane-Julie example. To prevent it in general, we have to demand non-causal evidence of preemption. It is important to note that causes of $\phi$ can certainly contribute to blocking putative causes of $\phi$. However, to avoid circularity and make sense of how preemption is related to overdetermination, we need additional evidence in these cases.

This can be understood as a demand directed at the modeller; to recognise preemption, we need models that are sufficiently fine-grained to represent the {\it mechanism} by which a putative cause has been blocked. This is exactly what we saw in the Billy-Suzy case, where the standard {\halpern} model needs to include an auxiliary variable, $H_B$, to reach the result that Billy has been blocked (a more naive model will regard both Suzy and Billy as causes). There is no reason why the same demand should not be made in the Jane-Julie cases; if there is really some mechanism whereby Julie's larger dose of poison renders Jane's dose causally insignificant, then the modeller has to include this information in the form of non-causal evidence of non-causation.

We are now ready to formalise the requirement that preemption requires proof, a principle we will refer to as {\it presumption}. It is essentially a principle of default reasoning; unless there is independently verifiable evidence to the contrary, a putative cause should always be regarded as a cause.

\begin{definition}\label{The presumption principle}
A causal theory $\cause$ satisfies the principle of presumption if for all models $M$, all formulas $\phi$ and all variables $V$, we have the following:
\begin{itemize}
\item If $V$ is a preempted cause of $\phi$, then for all $M' \in {\cal M}(\cause(M,\phi))$ such that $S_{M'}(V) = S_M(V)$ and $V \in \butfor(M',\phi)$, there is $W \not \in \cause(M, \phi)$ such that $S_M(W) \not = S_{M'}(W)$.
\end{itemize}
\end{definition}

The evidence of preemption that we need is provided by the $W$'s in the definition above; we need one for every similarity-based counterfactual where $V$ comes out as a but-for cause. The actual value of $W$ effectively blocks the causal contribution of $V$ in the actual situation, relative to at least one counterfactual where $V$ is a but-for cause of $\phi$. Or, to put it differently: no matter how we modify known causes of $\phi$ (while holding some non-causes fixed), as soon as we modify so many that $V$ becomes a but-for cause, at least one non-causal $W$ will also have changed, witnessing to $V$ having been blocked in the actual situation.

The example of Jane and Julie in model $M_2$ shows that the {\halpern} theory does not satisfy the presumption principle. As mentioned in the introduction, the predecessors of the {\halpern} theory judge this example differently, satisfying the principle of presumption in this case. A more complicated example for which the previous version of {\halpern}-causality also violates presumption is given below.

\begin{example}\label{ex:2}
Imagine that Jane is deciding whether to poison Bob, while June and Julie are deciding independently whether to give him a dose of antidote, just in case. Unfortunately for Bob, the antidote has a peculiar property: two doses behave like one dose of poison. As it happens, Jane gives Bob poison, while both June and Julie give him an antidote. To model this, we can use $M$ depicted below, where all variables are binary-valued and the context is given by $\vec u = (1,1,1)$.
$$
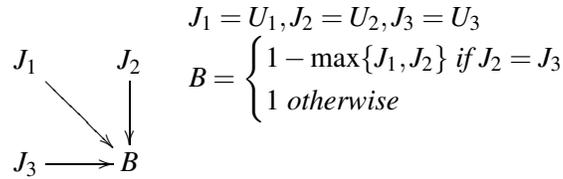
\begin{array}{cc}
\xymatrix{ J_1 \ar[dr]& J_2 \ar[d] \\
	J_3 \ar[r] & B } & \begin{array}{ll}
J_1 = U_1, J_2 = U_2, J_3 = U_3 \\
B = \begin{cases} 1 - \max\{J_1,J_2\} \text{ if } J_2 = J_3 \\
					1 \text{ otherwise }
	\end{cases}
\end{array}
\end{array}
$$
So $M \models (B = 0)$, since $S_M(J_1) = S_M(J_2) = 1$ so that $S_M(B) = 1 - S_M(J_1) = 0$. Moreover, we have $M \models [\assign{J_2}{0}]\neg (B = 0)$ and $M \models [\assign{J_3}{0}] \neg (B = 0)$. So both $J_2$ and $J_3$ are but-for causes of $(B = 0)$ in $M$. However, observe that we have $M \models [\assign{J_1}{0}](B = 0)$, since $1 = max\{J_1,J_2\}$ is still the case when $J_1$ is assigned $0$. It follows that $J_1$ is not a but-for cause of $(B = 0)$. Hence, by the minimality restriction, $J_1$ is not an $\halpern$ cause either. But this cannot be right. We know $J_1$ has the same effect as $J_2$ and $J_3$ combined in this case, so it stands to reason that $J_1$ should also be regarded as making a causal contribution to $(B = 0)$. Indeed, the antidotes do not work, since they behave like poison.
\end{example}

The causal judgements prescribed for this example seem hard to justify. Jane gives Bob a deadly dose of poison, but she is still not considered a cause of his death. The reason for this cannot be that the two antidotes preempt Jane's contribution. Two antidotes behave like a single dose of poison, they do not trump it. Rather, one seems forced to conclude that Jane is not considered a cause because June and Julie {\it could have} coordinated their actions in such a way that Jane's poison would not have been deadly. But this cannot be a sound form of causal reasoning; the mere possibility of prevention should not block causal attributions. However, as Example \ref{ex:2} shows, there are cases when the minimality constraint of the $\halpern$ definition does exactly that. Similarly, the invariance constraint that is used to define causality in \cite{halpern05}, the previous version of {\halpern} causality, yields the same result. The very first {\halpern} definition from \cite{halpern01} manages to get things right in this case, but other problems arise, showing that this definition cannot be accepted; it is simply too liberal about what it regards as causes \cite{hopkins03}.
\noo{
In the next section, I consider the converse of presumption, leading to a fixed-point characterisation of causal theories that I call {\it empirical}. In addition, I present some technical results that are useful when exploring this class of causal theories.

\section{Empirical Theories, Some Preliminary Results}

Arguably, causal theories in the {\halpern} tradition ought to be similarity-based. Moreover,}

To my knowledge, there is no widely endorsed theory of causality for the HP framework that satisfies the principle of presumption. What about the converse of presumption: is it reasonable to demand that putative causes should not be regarded as actual causes when non-causal blockers can be identified? In my opinion, this is a reasonable requirement at the theoretical level, as long as the quality of evidence is not in question. This position leads naturally to the following fixed-point characterisation, picking out a class of causal theories.

\begin{definition}\label{def:empirical}
A causal theory $\cause$ is {\it empirical} whenever the following holds, for all $M, \phi, V$:
$$
\begin{array}{l}
V \in \cause(M, \phi) \Leftrightarrow \exists M' \in {\cal M}(\cause(M, \phi)): \\ \big( S_M(V) = S_{M'}(V) \text{ and } V \in \butfor(M', \phi) \text{ and } \forall W \not \in \cause(M, \phi): S_M(W) = S_{M'}(W) \big)
\end{array}
$$
\end{definition}

We record the following simple observation about empirical theories (proof omitted).

\begin{proposition}\label{prop:suff}
If $\cause$ is an empirical causal theory, then it is similarity-based and it satisfies the principle of presumption.
\end{proposition}

It remains for future work to investigate further the class of empirical causal theories. Suffice it to say that in my opinion, Proposition \ref{prop:suff} summarises why the class is worth exploring. 

\section{A Remark on Asymmetry}\label{comp}

I have not proposed a new definition of actual causality. Instead, I have formulated {\it principles} that I believe any theory of actual causality (in the HP tradition) should satisfy. This is reminiscent of the approach taken in \cite{beckers16}, where several principles of causality are defined, before the authors go on to propose a new definition of causality that satisfy them. However, while the methodology is similar, the outcome is quite different. In \cite{beckers16}, the authors focus on the concept of {\it production}, generalising the version of this concept introduced in \cite{hall04}. This involves addressing how actual causation depends on time, leading to an extension of the HP framework that includes the formal notion of a ``timing''.

While the relationship between time and causation is no doubt interesting, it is not dealt with explicitly in the present paper, which sticks with the standard HP formalism. More importantly, the key principle postulated in \cite{beckers16} -- referred to as {\it asymmetry} -- stands in opposition to the principle of presumption postulated above. The asymmetry principle captures a specific intuition one might have about certain kinds of models: if $V$ is a binary-valued variable and the outcome we are looking at is $X$, then it should not be the case that intervening only to assign a different value to $V$ results in a model where $V$ is still a cause of $X$. That is, if $V$ counts as an actual cause of $X$ it should no longer count as a cause when we change the value of $V$, as long as we do not intervene on any other variable (so the equations and the context will deteremine the value of other variables). 

The principle of asymmetry is hard to accept in general. If Jane gives Bob half a dose of poison A while Julie gives Bob half a dose of poison B, a person choosing whether to give Bob half a dose of either poison A or poison B will typically come out as a cause of Bob's death either way (even if they are not to blame). Moreover, while it is easy to agree with the authors of \cite{beckers16} that ``the absence of a cause fulﬁlls a different role than the cause itself'', this observation is hardly a justification of asymmetry as defined in that paper. Causes are found by reasoning counterfactually, but the asymmetry principle does not permit us to perform any interventions when looking for evidence of the ``different role'' played by the other value of $V$. By contrast, if the intuition referred to in \cite{beckers16} had been formalised by requiring that asymmetry must be a {\it relevant possibility} whenever $V$ is a cause of $X$, any similarity-based causal theory would trivially satisfy it. After all, similarity-based causal theories satisfy something stronger, namely that $V$ is a but-for cause of $X$ in some relevant counterfactual model, c.f., Definition \ref{def:sim}.


The authors of \cite{beckers16} also argue for the asymmetry principle by relating it to examples where it has been claimed that the HP theory behaves unintuitively. One such example, closely related to the recalcitrant voter mentioned in Section \ref{sec:3}, is the so-called Switch: one imagines two railway tracks leading to the same station and a switch that determines which track the train will use. Is the switch a cause of the train arriving at the station? Intuitively, many would answer that it is not a cause. Since the train arrives regardless of which track it uses, there is no clear evidence of difference-making in this model. However, under the HP theory, the switch is a cause of the train arriving. This attribution clearly violates the asymmetry principle: this switch is a cause {\it regardless} of its value (and regardless of which track the train actually uses).

It is not so obvious that the HP theory gets the Switch wrong. As long as we are not trying to model laws of nature, I am inclined to say the judgement provided is correct. The underlying question raised by the Switch is the following: what counterfactual interventions should be permitted when looking for evidence of difference-making? As I mentioned in Section \ref{sec:3}, this is a question for which the HP theory provides a rather permissive answer. Specifically, by allowing us to intervene at will on causal variables while holding any subset of non-causal variables fixed, the HP theory allows us to instantiate value assignments that are {\it impossible} according to the structural equations in the model (that is, assignments not instantiated by any context). For instance, the witness to the switch being a cause of the train arriving at the station includes the railway track that is not actually used. By insisting that this track is not used even after the switch is changed (permitted by Definition \ref{def:hpcause}), we arrive at a model where the train does not arrive and the switch made the difference. This only makes sense if we allow ourselves to depart from the model by considering additional contingencies. What if, for instance, a tree blocks the railway track that was not used? In this case, the switch clearly makes a causal contribution to the train arriving, as pointed out in \cite{halpern05}.

The question is whether we should be permitted to question the model in this way. If the answer is yes, as in the HP theory, then the asymmetry principle must be rejected. The two railway tracks are then symmetric in the sense that there is the {\it possibility} that they could have been blocked, even if this is not recorded by our model. Hence, the switch is a cause of the train arriving, regardless of its value. The authors in \cite{beckers16} are aware of this objection and deal with it by proposing also a weaker version of their asymmetry principle. This version effectively concedes that the principle does not hold without exception, but requires us to mark exceptions explicitly by introducing special ``non-deterministic'' variables whose values are undetermined whenever asymmetry fails.\footnote{If the context is such that both values of $V$ are causes of $X$, the weak symmetry principle states that there has to be a choice of values for $V$ and the undetermined variables such that $V$ is not a cause of $X$. In this way, a larger space of counterfactuals is introduced through the back door, as it were.} In my opinion, the case for asymmetry remains unconvincing. However, enriching the HP model by allowing partial variable assignments seems potentially appealing. Specifically, it might offer a path to a reasonable theory of causality that only permits counterfactual value assignments that are instantiated by some (partial) context, leaving all equations intact (so that interventions cannot lead to ``impossible'' states).

\noo{As it stands, the HP theory does not satisfy this constraint, effectively forcing us to entertain the possibility that the scenario we consider may have been modelled incorrectly. In my opinion, this is appropriate in most contexts, especially in view of how the uncertainty is clearly restricted: we are only allowed to consider the possibility that observed dependencies (equations) might fail, we cannot introduce {\it new} dependencies {\it post hoc} to get the causal judgement we want. Moreover, an appeal to uncertainty about the model is required already in the standard Suzy-Billy example, where we cannot conclude that Billy has been preempted unless we entertain the ``impossible'' counterfactual state where he throws and misses, leaving the bottle intact. When looked at in this way, the Switch does not raise any concerns that do not already arise (or ought to arise) with respect to the core examples of preemption that are used to {\it motivate} the HP account of causation. In short, the HP theory denies that structural equations can be necessarily accurate. For this reason, I should not like to take the present version of the HP theory as a point of departure when discussing laws of nature (assuming such exist).} Even without such a modification, the HP framework makes sense when we work with particular facts, and the equations in our system are to be regarded as no more than records of observed (or defeasibly stipulated) dependencies. Still, as examples like the Switch show, I should not like to take the present version of the HP theory as a point of departure when discussing laws of nature (assuming such exist). This caveat also explains why I have been interested in finding a principle of causation that provides for a well-behaved {\it empirical} theory, as opposed to an axiomatic or metaphysical one.

\section{Conclusion}\label{sec:5}

In this paper, I have given a formal definition of preemption and overdetermination for causal theories in the {\halpern} formalism. When defined in a precise way, these two notions seem very useful as means for evaluating, comparing and classifying causal theories. Indeed, this is how the notions are already used in the literature, where they are regularly invoked in informal discussions about formal and semi-formal theories of actual causality.\footnote{See, e.g.,  \cite{lewis73,hitchcock11,paul13,bernstein14}.} By using a formalisation, I was able to identify problems with how the distinction between preemption and overdetermination is made in leading theories of causality formulated in the {\halpern} tradition. 

To arrive at the key definition, I used the most recent {\halpern} definition of causality to identify a notion of putative causation based on counterfactual intervention and but-for testing. This led to a definition of preempted causes as putative causes that are not recognised as actual ones. The problematic cases were then diagnosed as resulting from the minimality constraint used by the {\halpern} definition as an abstract way of forcing us to regard some putative causes as having been preempted.

While it is true that not all putative causes can be recognised as actual causes, I argued that the minimality constraint lacks general appeal and has little or no explanatory power when it comes to identifying cases of preemption. Essentially, my point was that the minimality constraint is a remnant of logic-based ideas about causality that we know to be flawed and ought to abandon completely. I then proposed the principle of presumption, whereby a putative cause must be recognised as an actual cause unless there is theory-independent evidence to the contrary. 

Following up on this, I proposed the principle of empiricism, defining thereby a class of causal theories that always recognise exactly those putative causes for which there is no independent evidence of preemption. The definition amounts to providing a fixed-point characterisation of causal theories that rely entirely on but-for testing and counterfactual intervention when deciding what counts as a cause. It follows from the examples given in this paper that the leading {\halpern} definitions of causality are not empirical theories. In future work, I will present results that shed further light on the class of empirical causal theories. 

\bibliography{aaai}
\bibliographystyle{eptcs}
\end{document}